\documentclass{article}

\usepackage{PRIMEarxiv}

\usepackage[utf8]{inputenc} 
\usepackage[T1]{fontenc}    
\usepackage{hyperref}       
\usepackage{url}            
\usepackage{booktabs}       
\usepackage{amsfonts}       
\usepackage{nicefrac}       
\usepackage{microtype}      
\usepackage{lipsum}
\usepackage[ruled,linesnumbered]{algorithm2e}
\usepackage{fancyhdr}       
\usepackage{graphicx}       
\usepackage{url}
\graphicspath{{media/}}     
\usepackage{amsmath}
\usepackage{amsthm}
\theoremstyle{definition}
\newtheorem{definition}{Definition}
\newtheorem{theorem}{Theorem}
\newtheorem{lemma}{Lemma}
\title{Edged Weisfeiler-Lehman algorithm
\thanks{\textit{\underline{Citation}}: 
\textbf{Yue, X., Liu, B., Zhang, F., Qu, G. (2024). Edged Weisfeiler-Lehman Algorithm. In: Wand, M., Malinovská, K., Schmidhuber, J., Tetko, I.V. (eds) Artificial Neural Networks and Machine Learning – ICANN 2024. ICANN 2024. Lecture Notes in Computer Science, vol 15020. Springer, Cham. https://doi.org/10.1007/978-3-031-72344-5\_7}} 

}

\author{
  Xiao Yue, Guangzhi Qu \\
  Oakland University\\
  Rochester, USA\\
  \texttt{\{xiaoyue, gqu\}@oakland.edu} \\
   \And
    Bo Liu \\
  Massey University \\
  Palmerston North, New Zealand\\
  \texttt{b.liu@massey.ac.nz} \\
     \And
    Feng Zhang \\
  China University of Geosciences \\
  Wuhan, China\\
  \texttt{jeff.f.zhang@gmail.com} \\
}

\begin{document}
\maketitle

\begin{center}
\small
This is the author’s accepted manuscript (AAM) of a paper published in ICANN 2024, Lecture Notes in Computer Science (LNCS), Springer.
The final authenticated version is available at: \url{https://link.springer.com/chapter/10.1007/978-3-031-72344-5_7}
\end{center}

\begin{abstract}
As a classical approach on graph learning, the propagation-aggregation methodology is widely exploited by many of Graph Neural Networks (GNNs), wherein the representation of a node is updated by aggregating representations from itself and neighbor nodes recursively. Similar to the propagation-aggregation methodology, the Weisfeiler-Lehman (1-WL) algorithm tests isomorphism through color refinement according to color representations of a node and its neighbor nodes. However, 1-WL does not leverage any edge features (labels), presenting a potential improvement on exploiting edge features in some fields. To address this limitation, we proposed a novel Edged-WL algorithm (E-WL) which extends the original 1-WL algorithm to incorporate edge features. Building upon the E-WL algorithm, we also introduce an Edged Graph Isomorphism Network (EGIN) model for further exploiting edge features, which addresses one key drawback in many GNNs that do not utilize any edge features of graph data. We evaluated the performance of proposed models using 12 edge-featured benchmark graph datasets and compared them with some state-of-the-art baseline models. Experimental results indicate that our proposed EGIN models, in general, demonstrate superior performance in graph learning on graph classification tasks. \url{https://github.com/YxRicardo/EGIN}
\end{abstract}


\section{Introduction}
With the development of graph learning, various domains including social analysis \cite{backstrom2011supervised}, biology, transportation and financial systems \cite{wu2020comprehensive} have embraced the integration with graph learning techniques. However, natural properties of lacking Euclidean structured data in graphs impede the applications of machine learning in these fields. To address learning with Non-Euclidean graph structure data, Graph Neural Networks (GNNs) \cite{micheli2009neural,scarselli2008graph} and their variants such as graph auto-encoder \cite{kipf2016variational}, graph recurrent neural networks \cite{li2015gated,tai2015improved}, graph attention networks \cite{velivckovic2017graph} and graph convolutional networks \cite{kipf2016semi} have been proposed. As a classical approach on graph learning, the propagation-aggregation methodology is widely exploited by many of GNNs, wherein the representation of a node is updated by aggregating representations from itself and neighbor nodes recursively. Consequently, a representation of the entire graph is derived through a Readout function, such as a pooling layer or an aggregation function.\\
Similar to the propagation-aggregation methodology, the Weisfeiler-Lehman (1-WL) algorithm \cite{weisfeiler1968reduction} tests isomorphism through color refinement according to color representations of a node and its neighbor nodes. The color representation of each node is updated by aggregating color representations from neighbor nodes in each iteration. This algorithm converges when color representations of all nodes are stabilized, generating a representation of the whole graph. Two WL representations are identical if two graphs are isomorphic. Despite it's one of the most efficient algorithms for isomorphism testing, several drawbacks such as only considering 1-tuple of the neighbor nodes and incapability on handling featured edges, impact performance of 1-WL to fail in distinguishing some simple examples \cite{huang2021short}. Inspired by the Weisfeiler-Lehman test, Xu et al. developed the Graph Isomorphism Network (GIN) \cite{xu2018powerful} to explore the maximum power of graph neural networks. However, similar to the one drawback of 1-WL, GIN does not leverage any edge features, presenting a potential improvement on exploiting edge features in some fields. For instance, in the context of the chemical graph theory, a chemical compound is described as a graph whose nodes correspond to atoms in the compound while edges represent chemical bonds. The chemical attributes and properties of the entire compound can be influenced by variations in chemical bonds. To address this limitation, we proposed a novel Edged Weisfeiler-Lehman algorithm (E-WL) that extends the original 1-WL algorithm by incorporating edge features. E-WL algorithm improves its ability to identify more isomorphisms by also considering representations of all neighbor edges when updating the color representation of a node. In this case, two graphs that have same topology structure and node features but different edge features will not be considered isomorphic. Building upon the E-WL algorithm, we also introduce an Edged Graph Isomorphism Network (EGIN) for further exploiting edge features, which addresses one key drawback in many GNNs that do not utilize any edge features of graph data. By fully exploiting edge features in graphs, discriminative powers of GNNs, which refer to abilities of the models to effectively distinguish and classify different graphs, are enhanced on edge-featured graph data. In addition, to address a potential limitation due to edge feature dimensions on EGIN, we have developed two variants of EGIN which improve the update and aggregation methods, denoted as EGIN-C (Cross updating) and EGIN-E (Edge embedding). EGIN-C exploits a cross updating method to create the new representations of nodes, while EGIN-E utilizes embeddings of edge features instead of original edge features at each iteration. We evaluated the performance of proposed models using 12 edge-featured benchmark graph datasets and compared them with some baseline models. Experimental results show that our proposed EGIN models outperform baselines on the most of the datasets. Contributions of this paper are summarized as follows:
\begin{itemize}
    \item We extend the Weisfeiler-Lehman algorithm by considering edge features in graphs, leading to a new Edged Weisfeiler-Lehman Algorithm.
    \item Based on the proposed Edged Weisfeiler-Lehman Algorithm, we develop a novel Edged Graph Isomorphism Network for effectively exploiting edge features in graph data.
    \item We propose two variants of EGIN which address the potential limitation of feature dimensions on the original EGIN.
\end{itemize}

\section{Preliminaries}
This section introduces preliminaries, covering notations, the Weisfeiler-Lehman algorithm, and a general model of Graph Neural Networks based on propagation-aggregation
strategy. 

\subsection{Notations}
Given an undirected graph $\mathcal{G} = (\mathcal{V},\mathcal{E}, \mathcal{X}_{V},\mathcal{X}_{E})$ assumed to be free of self-loops and isolated nodes, where $\mathcal{V}$ represents the set of all nodes and $\mathcal{E}$ represents the set of all edges in $\mathcal{G}$, $\mathcal{X}_{V}$ and $\mathcal{X}_{E}$ are sets of node features and edge features, respectively. Each node $n_{i}\in \mathcal{V}$ has a feature vector denoted as $x^{n}_{i} \in \mathcal{X}_{V}$. Each edge $e_{i,j} \in \mathcal{E}$ connecting nodes $n_{i}$ and $n_{j}$ has a feature vector denoted as $x^{e}_{i,j} \in \mathcal{X}_{E}$. For graphs without any node features, $x^{n}_{i}$ can be represented by the node degree, a constant, or a unique identifier. In this work, our emphasis is on edge-featured graphs, as edge features play a crucial role in the target applications of interest. Otherwise, the proposed approaches may degenerate into some existing methods.
\begin{definition}
  Isomorphism: Given two graphs $\mathcal{G}^{1} = (\mathcal{V}^{1},\mathcal{E}^{1}, \mathcal{X}^{1}_{V},\mathcal{X}^{1}_{E})$ and $\mathcal{G}^{2} = (\mathcal{V}^{2},\mathcal{E}^{2}, \mathcal{X}^{2}_{V},\mathcal{X}^{2}_{E})$, $\mathcal{G}^{1}$ and $\mathcal{G}^{2}$ are isomorphic if there exists a bijective mapping $f$ between $\mathcal{V}^{1}$ and $\mathcal{V}^{2}$, denoted as $f:\mathcal{V}^{1} \rightarrow \mathcal{V}^{2}$. If $e_{i,j} \in \mathcal{E}^{1}$, then $e_{f(i),f(j)} \in \mathcal{E}^{2}$. If $e_{i,j} \notin \mathcal{E}^{1}$, then $e_{f(i),f(j)} \notin \mathcal{E}^{2}$.
\end{definition}

\subsection{Weisfeiler-Lehman Algorithm}
\begin{definition}
  Multiset: A multiset is a set allowing for multiple instances of each of its elements, noted as $\{\{\cdot\}\}$. 
\end{definition}

The Weisfeiler-Lehman algorithm is widely adopted for isomorphism testing. Consider a graph $\mathcal{G} = (\mathcal{V},\mathcal{E}, \mathcal{X}_{V},\mathcal{X}_{E})$, where $\mathcal{X}_{V}$ and $\mathcal{X}_{E}$ are also multisets, the color representation of a node $n_{i}$ at the $l^{th}$ iteration is denoted as $c^{(l)}_i$. Here we assume all node features $x^{n}_{i}$ are discrete. Initially, $c^{(0)}_i$ is assigned to be equal to $x^{n}_{i}$. The algorithm keeps updating the color representation $c^{(l)}_i$ according to Equation \ref{1-wl} iteratively until it arrives at convergence, where \emph{HASH} represents an injective mapping that maps $c^{(l-1)}_i$ to $c^{(l)}_i$ and $\mathcal{N}_{(i)}$ indicates the set of all neighbor nodes of node $n_{i}$. In other words, $c^{(l)}_i$ is unique for each unique pair $(c^{(l-1)}_i, \{\{c^{(l-1)}_j| n_{j} \in \mathcal{N}_{(i)}\}\})$.
\begin{equation}
\label{1-wl}
    c^{(l)}_i = HASH(c^{(l-1)}_i, \{\{c^{(l-1)}_j| n_{j} \in \mathcal{N}_{(i)}\}\})
\end{equation}
Hence Weisfeiler-Lehman algorithm updates color representations of nodes $c^{(l)}_i$ by considering its own previous color representation $c^{(l-1)}_i$ and an aggregation of color representations from a multiset of their neighbor nodes $\{\{c^{(l-1)}_j| n_{j} \in \mathcal{N}_{(i)}\}\}$. The final color representations of all nodes are derived when the color representations between two iterations remain unchanged, meanwhile a representation of the whole graph is obtained. Weisfeiler-Lehman algorithm determines that two graphs are not isomorphic if they have different graph representations. It has been successful in passing isomorphism tests for the majority of graphs \cite{babai1979canonical} although there exist some failure cases \cite{sato2020survey}.
\subsection{Propagation-Aggregation Strategy in Graph Neural Networks}
As a prevailing strategy adopted by majority of GNNs, propagation-aggregation strategy iteratively updates representations of nodes by aggregating node representations from neighbor nodes, thereby propagating information from nodes to nodes. Let $a^{(k)}_{i}$ describe an aggregation of representations from neighbor nodes of $n_{i} \in \mathcal{V}$. Additionally, let $h^{(k)}_{i}$ denote a latent representation of the node $n_{i}$ at the $k^{th}$ GNN layer and $h^{(0)}_{i} = x^{n}_{i}$. The AGGREGATE and UPDATE functions at the $k^{th}$ GNN layer are performed according to the Equations \ref{agg1} and \ref{update}, respectively. 
\begin{equation}
\label{agg1}
    a^{(k)}_{i} = AGGREGATE^{(k)}(\{\{h^{(k-1)}_{j}| n_{j} \in \mathcal{N}_{(i)}\}\})
\end{equation}
\begin{equation}
\label{update}
    h^{(k)}_{i} = UPDATE^{(k)}(a^{(k)}_{i}, h^{(k-1)}_{i})
\end{equation}
GNNs output a multiset of representations of all nodes $\{\{h^{(k)}_{i}| n_{i} \in \mathcal{V}\}\}$ in the final layer as node-level embeddings. For graph-level task, a READOUT function is utilized to construct a graph presentation $h_{G}$ by synthesizing representations of all nodes, as shown in Equation \ref{readout}.
\begin{equation}
\label{readout}
    h_{G} = READOUT(\{\{h^{(k)}_{i}| n_{i} \in \mathcal{V}\}\})
\end{equation}
Three common choices of READOUT functions are: \textit{Max}, \textit{Mean} and \textit{Sum}, with \textit{Sum} demonstrating the highest expressive power \cite{xu2018powerful}.

\section{Related work}
\subsection{Improvements on 1-WL Algorithm}
As a widely employed algorithm for testing isomorphisms, the classical 1-WL algorithm has recently drawn attentions for enhancing the performance of graph neural networks. However, the limitation of failing to identify some special cases \cite{huang2021short} highlights the need to extend the 1-WL to higher orders. A higher-order extension, known as k-WL \cite{grohe2015pebble,grohe2017descriptive} has been proposed to update the node representations using k-tuples, which consist of k nodes. Morris et al. proposed k-folklore-WL (k-FWL) \cite{morris2019weisfeiler}, which is based on higher-order graph structures at multiple scales. Both k-WL and k-FWL have been mathematically proved to enhance the discriminative power \cite{huang2021short}. 
\subsection{Graph Neural Networks}
Graph Neural Networks primarily utilize the propagation-aggregation strategy to iteratively update a node's representation by aggregating information from its neighboring nodes. The representation of the entire graph is derived by applying a READOUT function on representations of all the nodes in the graph. Among classical GNN variants, graph convolutional networks (GCNs) have shown outstanding performance on graph learning and can be broadly categorized into two main types \cite{zhang2019graph}: spectral-based and spatial-based. Spectral-based graph convolutional networks \cite{bruna2013spectral,henaff2015deep,kipf2016semi,defferrard2016convolutional} are designed to apply convolutional operations on graphs in Fourier domain. To overcome the limitations of fixed graph structures on convolutional operations in Fourier domain, spatial-based graph convolutional networks \cite{duvenaud2015convolutional,niepert2016learning,hamilton2017inductive,zhang2018end,li2020deepergcn} have been developed. These networks bypass computational intensive convolutional operations by replacing them with the propagation-aggregation strategy. In addition to GCNs, Graph attention networks (GATs) \cite{velivckovic2017graph} exploit attention mechanisms with a convolution-style approach by leveraging masked self-attentional layers. Graph isomorphism networks \cite{xu2018powerful} are designed to maximize the discriminative power of a GNN based on the Weisfeiler-Lehman Algorithm. Xu et al. proposed GraphSAGE \cite{xu2020inductive}, which employs inductive representation learning by sampling neighborhoods and aggregating information. Graph neural networks, specifically developed for graph-related tasks such as node classification and graph classification, have demonstrated outstanding performance across various applications.

\section{Edged Weisfeiler-Lehman Algorithm}
In this section, we introduce the Edged Weisfeiler-Lehman (E-WL) algorithm, which serves as an extension of 1-WL algorithm. The E-WL algorithm addresses a limitation of 1-WL algorithm by incorporating edge features. To achieve  this, we propose the Node-Edge tuple, a new data structure, and introduce a new representation update mechanism upon Node-Edge tuples.
\begin{definition}
  Node-Edge tuple: A Node-Edge tuple consists of two elements in a fixed order: the representation $c_{i}$ of a node $n_{i}$, and the representation $x^{e}_{i,j}$ of an edge $e_{i,j}$, denoted as $(c_{i}, x^{e}_{i,j})_{tup}$.
\end{definition}
 Note that a Node-Edge tuple must only contain related elements. For example, in the tuple $(c_{i}, x^{e}_{j,k})_{tup}$, the edge representation $x^{e}_{j,k}$ refers to the edge $e_{j,k}$ which is not directly connected to node $n_{i}$, making it undefined. For any node in a graph, there is always at least one edge connecting to a neighboring node if the graph has no isolated nodes. This forms at least one Node-Edge tuple. We define this specific type of multiset as a \emph{Node-Edge tuple multiset}, denoted as $T$. Given a node $n_{i}$ in graph $\mathcal{G}$, then $T_{i}=\{\{(c_{j},x^{e}_{i,j})_{tup}| n_{j} \in \mathcal{N}_{(i)}\}\}$, where $c_{j} \in \mathcal{X}_{V}, x^{e}_{i,j} \in \mathcal{X}_{E}$. The characteristics of Node-Edge tuple multisets are similar to standard multisets, making them applicable to \emph{deep multisets} theory \cite{zaheer2017deep,xu2018powerful}, which parameterizes universal functions with neural networks on multisets.\\
The proposed E-WL Algorithm is detailed in algorithm \ref{E-WL algorithm}. The input to the E-WL algorithm is a graph $\mathcal{G}$, and the output of the algorithm is a node-multiset $\mathcal{X}^{*}_{V}$. Initially, the representation $c^{(0)}_i$ of a node $n_{i}$ is set to the node features $x^{n}_{i}$. At each iteration, for each node $n_{i}$ in graph $\mathcal{G}$, a Node-Edge tuple multiset $T_{i}^{(l)}$ aggregates the neighbor Node-Edge tuples (line 6). The representation $c^{(l)}_i$ is then updated at line 7 based on its previous representation $c^{(l-1)}_i$ and the Node-Edge tuple multiset $T_{i}$. An updated node-multiset $\mathcal{X}^{(l)}_{V}$ of all node representations is derived for next iteration at line 8. This process converges until $c^{(l)}_i = c^{(l-1)}_i$ for all nodes in this graph, and $\mathcal{X}^{*}_{V}$ is returned as the output.\\
\begin{algorithm}[htb] 
	\caption{Edged Weisfeiler-Lehman algorithm (E-WL)}
	\label{E-WL algorithm}
\textbf{Input: $\mathcal{G} = (\mathcal{V},\mathcal{E}, \mathcal{X}_{V}, \mathcal{X}_{E})$}\\
 $c^{(0)}_i = x^{n}_{i}$ for all $n_{i} \in \mathcal{V}$\\
 $\mathcal{X}^{(0)}_{V} = \mathcal{X}_{V}$\\
    \textbf{repeat}\\
        \Indp
        for each $n_{i} \in \mathcal{V}$\\
        \Indp
        $T_{i}^{(l)} =  \{\{(c^{(l-1)}_{j},x^{e}_{i,j})_{tup}| n_{j} \in \mathcal{N}_{(i)}\}\}, c^{(l-1)}_{j} \in \mathcal{X}^{(l-1)}_{V}, x^{e}_{i,j} \in \mathcal{X}_{E}$\\
        $c^{(l)}_i = HASH(c^{(l-1)}_i, T_{i}^{(l)})$\\
        \Indm
        $\mathcal{X}^{(l)}_{V} =  \{\{c^{(l)}_i|n_{i} \in \mathcal{V}\}\}$\\
        \Indm
    \textbf{until} $c^{(l)}_i = c^{(l-1)}_i$ for all $n_{i} \in \mathcal{V}$\\
    \textbf{return} $\mathcal{X}^{*}_{V} = \mathcal{X}^{(l)}_{V}$
\end{algorithm}\\
Similar to the 1-WL algorithm, which constructs injective mappings based on a pair consisting of the a node's representation and a multiset of its neighboring nodes' representations, the E-WL algorithm updates a node's representation using an injective mapping formed from the node's own representation and a Node-Edge tuple multiset. Figures \ref{example1} and \ref{example2} illustrate the differences in discriminative power between the 1-WL and E-WL algorithms in isomorphism testing on the same set of three graphs. Figure \ref{example1} shows that the 1-WL algorithm identifies all three graphs as isomorphic since it does not consider edge features. In contrast, the E-WL algorithm can distinguish that graph \emph{(c)} is not isomorphic to graphs \emph{(a)} and \emph{(b)} by incorporating the edge features in node representation updates. Therefore, if edge features are available, the discriminative power of the E-WL algorithm surpasses that of the 1-WL algorithm.

\begin{theorem}
\label{t1}
    The Discriminative power of the E-WL algorithm is either equal to or greater than that of the 1-WL algorithm.
\end{theorem}

\begin{proof}
    We prove Theorem \ref{t1} from two scenarios: \emph{1.} if edge features are not available, then the discriminative power of E-WL is equivalent to that of 1-WL; \emph{2.}  if edge features are available, the discriminative power of E-WL is greater than that of 1-WL.\\
    case \emph{1:} The $x^{e}_{i,j}$ element in a Node-Edge tuple $(c^{(l-1)}_{j},x^{e}_{i,j})_{tup}$ is an edge representation. If edge features are unavailable, we denote the edge representations as a constant \emph{b} without losing generality. Thus, each Node-Edge tuple $(c^{(l-1)}_{j},x^{e}_{i,j})_{tup}$ becomes equivalent to a regular node representation $c^{(l-1)}_{j}$. Therefore, the E-WL algorithm is as powerful as the 1-WL in this cases. \\
    case \emph{2:} We have presented a case where E-WL is more powerful than 1-WL. To prove that 1-WL can't be more powerful than E-WL under any circumstance if edge features are available, we use a contradiction. Assuming the discriminative power of 1-WL is greater than that of E-WL in some cases. Consider two graphs $\mathcal{G}^{1} = (\mathcal{V}^{1},\mathcal{E}^{1}, \mathcal{X}^{1}_{V},\mathcal{X}^{1}_{E})$ and $\mathcal{G}^{2} = (\mathcal{V}^{2},\mathcal{E}^{2}, \mathcal{X}^{2}_{V},\mathcal{X}^{2}_{E})$, the following two conditions should hold at a certain $l^{th}$ iteration for $n_{i} \in \mathcal{V}^{1}, x^{e}_{i,j} \in \mathcal{X}^{1}_{E}$ and $n_{v} \in \mathcal{V}^{2}, x^{e}_{v,u} \in \mathcal{X}^{2}_{E}$.

\begin{equation}
    (c^{(l-1)}_i, \{\{c^{(l-1)}_{j}| n_{j} \in \mathcal{N}_{(i)}\}\}) \\ \neq (c^{(l-1)}_v, \{\{c^{(l-1)}_{u}| n_{u} \in \mathcal{N}_{(v)}\}\})
\end{equation}

\begin{equation}
\label{prove1-1}
    (c^{(l-1)}_i, \{\{(c^{(l-1)}_{j},x^{e}_{i,j})_{tup}| n_{j} \in \mathcal{N}_{(i)} \}\}) \\
    = (c^{(l-1)}_v, \{\{(c^{(l-1)}_{u},x^{e}_{v,u})_{tup}| n_{u} \in \mathcal{N}_{(v)}\}\})
\end{equation}

We can deduce from the Equation \ref{prove1-1} that $c^{(l-1)}_{i} = c^{(l-1)}_{v}$ and $\{\{(c^{(l-1)}_{j},x^{e}_{i,j})_{tup}|$ $ n_{j} \in \mathcal{N}_{(i)}, \}\} = \{\{(c^{(l-1)}_{u},x^{e}_{v,u})_{tup}| n_{u} \in \mathcal{N}_{(u)}\}\}$. If we remove edge representations from all Node-Edge tuples, we get $\{\{c^{(l-1)}_{j}| n_{j} \in \mathcal{N}_{(i)} \}\} = \{\{c^{(l-1)}_{u}| n_{u} \in \mathcal{N}_{(u)}\}\}$. Therefore, $(c^{(l-1)}_i, \{\{c^{(l-1)}_{j}| n_{j} \in \mathcal{N}_{(i)}\}\}) \neq (c^{(l-1)}_v, \{\{c^{(l-1)}_{u}| n_{u} \in \mathcal{N}_{(v)}\}\})$ can't hold in this case, leading to a contradiction. \\
\end{proof}

\begin{figure*}[!h]
    \centering
    \includegraphics[scale=0.41]{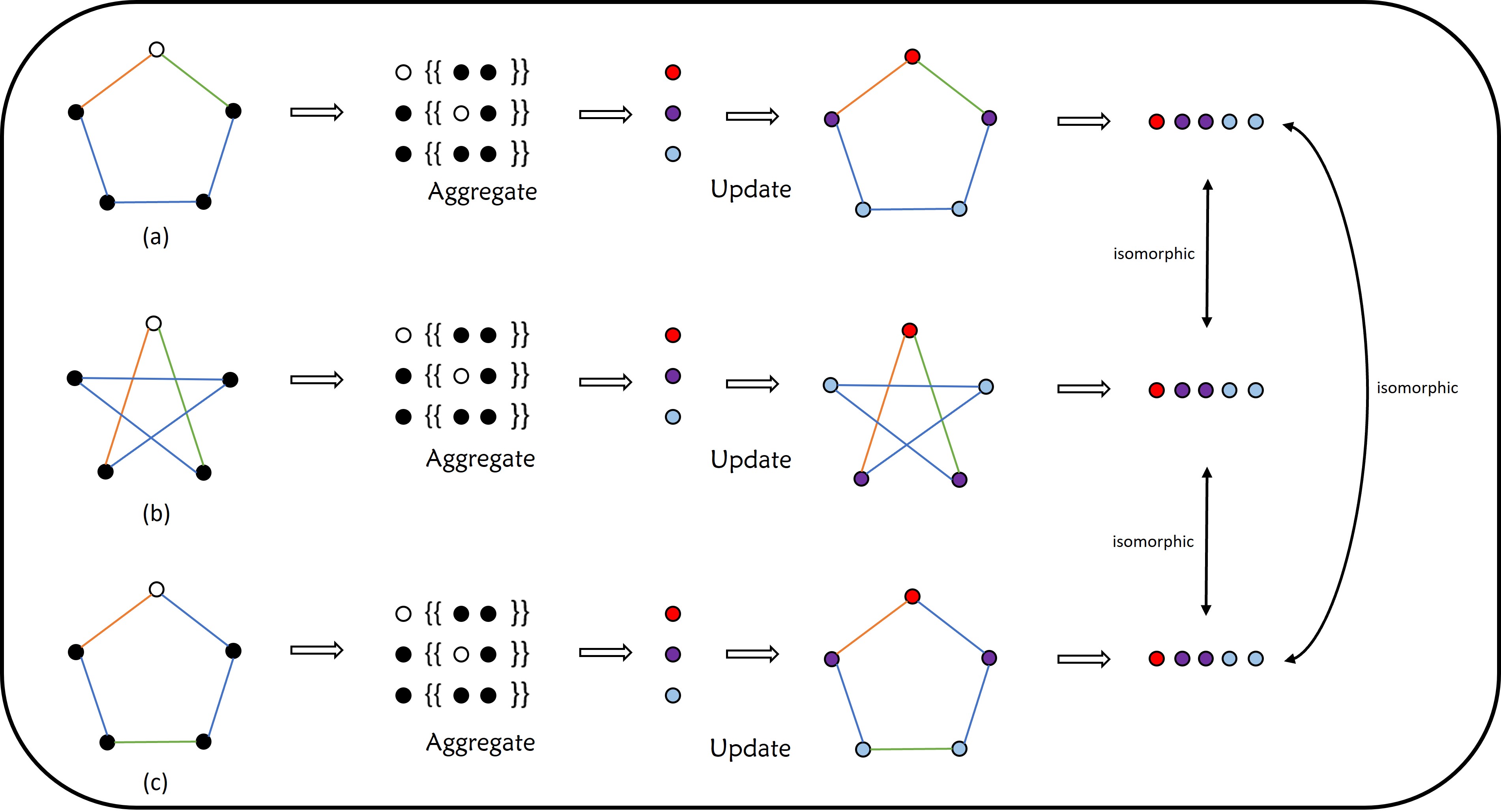}
    \caption{Apply 1-WL algorithm on three example graphs}
    \label{example1}
\end{figure*}
\begin{figure*}[!h]
    \centering
    \includegraphics[scale=0.41]{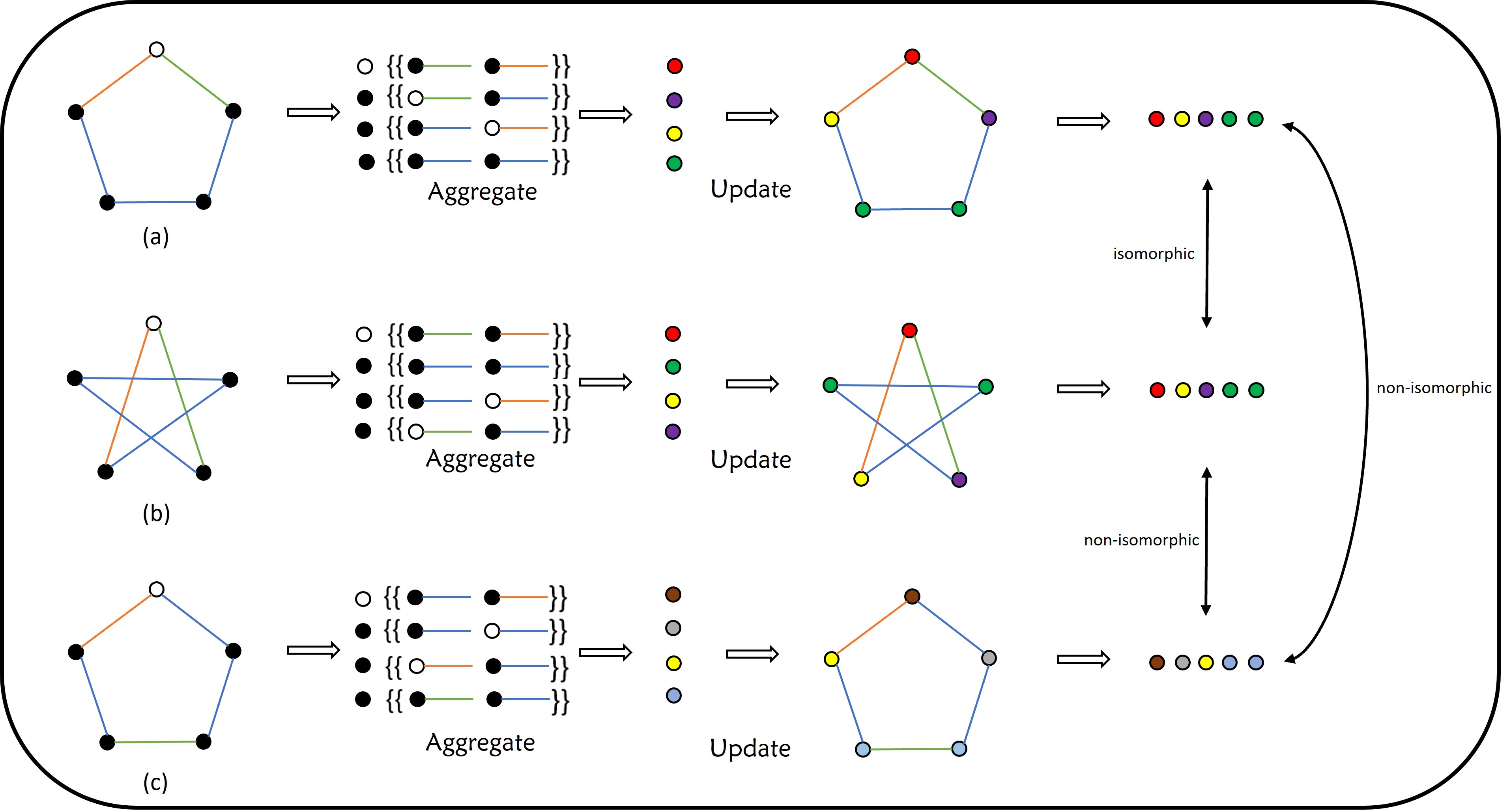}
    \caption{Apply E-WL algorithm on three example graphs}
    \label{example2}
\end{figure*}

\section{An Edged Graph Isomorphism Network Model}
A Graph Isomorphism Networks (GIN) is designed to map two graphs, determined to be isomorphic by 1-WL algorithm, to the same representation, while mapping non-isomorphic graphs determined by 1-WL algorithm to different representations. This ensures maximum discriminative power of graph neural networks, based on the 1-WL methodology. Similarly, we propose an Edged Graph Isomorphism Network (EGIN), extending the GIN model by integrating the E-WL algorithm. Given that the maximum discriminative power of any aggregation-based GNN is equivalent to the 1-WL testing \cite{xu2018powerful}, we assert that the proposed EGIN is as powerful as E-WL isomorphism testing in distinguishing different graphs with featured edges. We first formulate the aggregation and update processes of the EGIN model as shown in Equation \ref{agg2}, where $h_{i}^{(k)}$ denotes the latent representation of a node $n_{i}$ in the $k^{th}$ layer of the EGIN model. The function $f_{a}$ is an aggregation function, and the updating function is denoted as $\phi$. At each layer, $h_{i}^{(k)}$ is computed based on its representation $h_{i}^{(k-1)}$ from the previous layer and a multiset of Node-Edge tuples of its neighbors $\{\{(h^{(k-1)}_j,x^{e}_{i,j})| n_{j} \in \mathcal{N}_{(i)}\}\}$.  
\begin{equation}
\label{agg2}
    h_{i}^{(k)} = \phi(h_{i}^{(k-1)},f_a(\{\{(h^{(k-1)}_j,x^{e}_{i,j})| n_{j} \in \mathcal{N}_{(i)}\}\})
\end{equation}
To obtain the representation of the entire graph $g^{(k)}$ for graph-level tasks, we define a READOUT function $f_{r}$ in Equation \ref{read}.
\begin{equation}
\label{read}
    g^{(k)} = f_{r}(\{\{h_{i}^{(k)}| n_{i} \in \mathcal{V}\}\})
\end{equation}
Assuming aggregation function $f_{a}$, updating function $\phi$, and readout function $f_{r}$ are all injective, the EGIN model can approximate the maximum discriminative power of the E-WL algorithm. Given any two graphs $\mathcal{G}_{1}$ and $\mathcal{G}_{2}$, their graph representations $g^{(k)}_{1}$ and $g^{(k)}_{2}$ can be derived by the EGIN. If they are identified as non-isomorphic by the E-WL algorithm, then $g^{(k)}_{1} \neq g^{(k)}_{2}$ should hold. A \textit{sum} aggregator and an injective function $f_{ne}$ operating on Node-Edge tuples are leveraged in constructing the injective aggregation function $f_a(T)$ on a Node-Edge tuple multiset. In this way, the value of $f_a(T)$ is unique for each distinct input Node-Edge tuple multiset $T$. A viable choice of $f_{ne}$ is one-hot encoding that projects different tuples to distinct multi-digit representations. Consequently, $f_a(T)$ is defined as $f_a(T) = \sum_{(c,x^{e})_{tup} \in T} f_{ne}((c,x^{e})_{tup})$. We note here that the node representation $h_{i}^{(k-1)}$ in Equation \ref{agg2} is not a Node-Edge tuple, hence it can't serve as an input for the function $f_{ne}$. To address this issue, we define a special void edge element $\delta$ to build a Node-Edge tuple incorporating $h_{i}^{(k-1)}$. 
\begin{definition}
Void edge: a special edge with zero feature values in all dimensions, denoted as $\delta$.   
\end{definition}
The node representation $h$ severs as an additional input for function $f_a(T)$ by constructing a distinctive Node-Edge tuple that combines $h$ and $\delta$, denoted as $f_a(h,T) = f_{ne}((h,\delta)_{tup}) + \sum_{(c,x^{e})_{tup} \in T} f_{ne}((c,x^{e})_{tup})$. This value is unique for each pair of inputs $(h,T)$. Since $f_a(h,T)$ is permutation invariant, any function $G(h,T)$ on that input pair can be decomposed as $G(h,T) = \phi \{ f_a(h,T)\} =\phi \{f_{ne}((h,\delta)_{tup}) + \sum_{(c,x^{e})_{tup} \in T} f_{ne}((c,x^{e})_{tup})\}$. In GIN \cite{xu2018powerful}, a parameter $\epsilon$ is utilized as an indicator for node representation $h$ to maintain the injectivity of the composed functions. This indicator is not essential for the function $G(h,T)$ in EGIN, as the void edge $\delta$ already serves as an indicator to ensure the uniqueness of the pair of inputs $(h,T)$, since no Node-Edge tuple in $T$ should contain a void edge $\delta$. However, we still include this parameter in our model as an alternative method for constructing injective functions, and it can be removed by setting it to zero. Therefore, $G(h,T)$ is presented as: $G(h,T) = \phi \{(1+\epsilon) \cdot f_{ne}((h,\delta)_{tup}) + \sum_{(c,x^{e})_{tup} \in T} f_{ne}((c,x^{e})_{tup})\}$. The composed functions of $f_{ne}$ and $\phi$ in $G(h,T)$ can be learned through Multilayer Perceptrons (MLPs) due to their capacity to approximate any continuous function, as stated by the universal approximation theorem \cite{hornik1989multilayer}. As a result, the aggregation and update in each EGIN layer are represented in Equation \ref{EGIN}.
\begin{equation}
\label{EGIN}
     h_{i}^{(k)} = MLP^{(k)}\{(1+\epsilon^{(k)}) \cdot (h_{i}^{(k-1)}, \delta)    + \sum_{n_{j} \in \mathcal{N}_{(i)}} (h^{(k-1)}_j,x^{e}_{i,j})\}
\end{equation}
In Section \ref{experiments}, we refer to this model as EGIN-$\epsilon$, whereas the model without the $\epsilon$ parameter in the aggregation and update processes is denoted as EGIN. New representation of a node in the $k^{th}$ layer is updated based on both the representation of itself in the previous layer and a multiset of neighboring Node-Edge tuples. However, the EGIN model may lose generalization and degenerate into the original GIN model if the dimension of $h_{i}^{(k)}$ is significantly larger than that of $x^{e}_{i,j}$. For example, if the hidden dimension of $h_{i}^{(k)}$ is \emph{512} while the dimension of $x^{e}_{i,j}$ is only \emph{4} in a certain layer, node representations are likely to dominate during training. To address this limitation, we introduce two variants of EGIN: \emph{EGIN-C} (Cross updating) and \emph{EGIN-E} (Edge embedding).
\subsection{EGIN-C Model}
The first variant is denoted as EGIN-C with "C" indicating \emph{cross updating}. Given two vectors $h \in \mathbb{R}^{p}$ and $e \in \mathbb{R}^{q}$ , we define a cross updating operation, denoted as $\bigotimes(h,e) = \Vert^{p}_{i=0} \Vert^{q}_{j=0}(h_{i} \times e_{j}) = [h_{0} \times e_{0}, h_{0} \times e_{1}, ..., h_{0} \times e_{q}, h_{1} \times e_{0}, h_{1} \times e_{1}, ..., h_{1} \times e_{q}, ..., h_{p} \times e_{q}]$. In order to apply \emph{cross updating} operation on the node representation $h_{i}^{(k)}$, we define a unit edge.
\begin{definition}
Unit edge: an edge with unit feature values in all dimensions, denoted as $\psi$.
\end{definition}
We modified the function $f_a(h,T)$, so that the function $G(h,T)$ is denoted as $G(h,T) = \phi \{ f_a(h,T)\} = \phi \{(1+\epsilon) \cdot f_{ne}(\bigotimes(h,\psi)) + \sum_{(c,x^{e})_{tup} \in T} f_{ne}(\bigotimes(c,x^{e}))\}$, where $f_{ne}$ is still an injective function on the outputs of cross updating operation. Therefore, the aggregation and update processes of EGIN-C are formulated as shown in Equation \ref{cross}

\begin{equation} 
\label{cross}
    h_{i}^{(k)} = MLP^{(k)}\{(1+\epsilon^{(k)})\cdot \bigotimes(h_{i}^{(k-1)},\psi)  + \sum_{n_{j} \in \mathcal{N}_{(i)}} \bigotimes(h^{(k-1)}_j,x^{e}_{i,j})\}
\end{equation}

In Section \ref{experiments}, this model is referred to as EGIN-C-$\epsilon$, while the model without the $\epsilon$ parameter in the aggregation and update processes is denoted as EGIN-C. Note that the \emph{cross updating} is not an injective operation, and thus, the aggregation function
$f_{a}(h,T) = f_{ne}(\bigotimes(h,\psi)) + \sum_{(c,x^{e})_{tup} \in T} f_{ne}(\bigotimes(c,x^{e}))$ is also non-injective. This means it may not always produce unique outputs for each distinct pair of inputs $(h, T)$. However, it compensates for the loss of injectivity by integrating node and edge features. Consequently, performance of EGIN-C may not deteriorate in some cases, as the final task is graph classification rather than graph isomorphism identification. In this case, node representations are influenced by edge features when passed to neighboring nodes, leading to a synthesis of both node and edge features.

\subsection{EGIN-E Model}
Additionally, we proposed another variant of EGIN, called EGIN-E, where "E" stands for the embedding of edges, to address the aforementioned limitation. In this variant, We introduced another injective function $f_{e}$, which maps edge features $x^{e}_{i,j}$ to latent representations, into the function $G(h,T)$. The updated function is denoted as $G(h,T) = \phi\{f_a(h,T)\}= \phi \{f_{ne}((h,f_{e}(\delta))_{tup}) + \sum_{(c,x^{e})_{tup} \in T} f_{ne}((c,f_{e}(x^{e}))_{tup})\}$. Here, $(c,f_{e}(x^{e}))_{tup}$ is treated as a Node-Edge tuple, where $f_{e}(x^{e})$ constructs the representation of an edge. In this context, the function $f_{ne}$ can still be unique for each pair of inputs $(h,f_{e}(x^{e}))$. The aggregation and update processes of EGIN-E are formulated in Equation \ref{EGIN-E}, with the function $f_{e}$ represented by $MLP_{2}$.
\begin{equation}
\begin{split}
\label{EGIN-E}
h_{i}^{(k)} = MLP_{1}^{(k)}\{(1+\epsilon^{(k)}) \cdot (h_{i}^{(k-1)}, MLP_{2}^{(k)}(\delta)) \\
+  \sum_{n_{j} \in \mathcal{N}_{(i)}} (h^{(k-1)}_j,MLP_{2}^{(k)}(x^{e}_{i,j}))\}
\end{split}
\end{equation}

In Section \ref{experiments}, EGIN-E-$\epsilon$ refers to the model based on Equation \ref{EGIN-E}, while the model without the $\epsilon$ parameter is designated as EGIN-E. The $MLP_{1}$ functions similarly to the $MLP$ in Equations \ref{EGIN} and \ref{cross}, whereas the additional $MLP_{2}$ is used in each layer to project the edge features $x^{e}_{i,j}$ into a higher-dimensional space.

\section{Experiments}
\label{experiments}
\subsection{Datasets}

We evaluated the effectiveness of our proposed approaches by assessing their performance on multiple benchmark graph datasets for graph classification tasks. These datasets were sourced from the TU Graph dataset \cite{Morris+2020}, a comprehensive collection for graph learning tasks that includes a diverse range of graph datasets, such as small molecules, bioinformatics, computer visions, social networks and synthetic graphs. From this collection, we selected 12 datasets, all of which contain edge features (labels). These labeled edges were converted into feature vectors using a one-hot encoding mechanism. Detailed statistics for these datasets are provided in Table \ref{data}.

\begin{table*}[htb]
\small
	\centering
	\caption{Statistics of graph datasets}
		\label{data}
	\begin{tabular}{lllllll}
\hline
           & \begin{tabular}[c]{@{}l@{}}Number \\ of  Graphs\end{tabular} & Classes & \begin{tabular}[c]{@{}l@{}}Average \\ Nodes\end{tabular} & \begin{tabular}[c]{@{}l@{}}Average \\ Edges\end{tabular} & \begin{tabular}[c]{@{}l@{}}Node feature \\ dimension\end{tabular} & \begin{tabular}[c]{@{}l@{}}Edge feature \\ dimension\end{tabular} \\ \hline
MUTAG      & 188                                                          & 2       & 17.93                                                    & 19.79                                                    & 7                                                                 & 4                                                                 \\
AIDS       & 2000                                                         & 2       & 15.69                                                    & 16.20                                                    & 42                                                                & 3                                                                 \\
COX2\_MD   & 303                                                          & 2       & 26.28                                                    & 335.12                                                   & 7                                                                 & 6                                                                 \\
Cuneiform  & 267                                                          & 30      & 21.27                                                    & 44.80                                                    & 10                                                                & 4                                                                 \\
ER\_MD     & 446                                                          & 2       & 21.33                                                    & 234.85                                                   & 10                                                                & 6                                                                 \\
PTC\_FM    & 349                                                          & 2       & 14.11                                                    & 14.48                                                    & 18                                                                & 4                                                                 \\
PTC\_FR    & 351                                                          & 2       & 14.56                                                    & 15.00                                                    & 19                                                                & 4                                                                 \\
PTC\_MM    & 336                                                          & 2       & 13.97                                                    & 14.32                                                    & 20                                                                & 4                                                                 \\
PTC\_MR    & 344                                                          & 2       & 14.29                                                    & 14.69                                                    & 18                                                                & 4                                                                 \\
Tox21\_AhR & 607                                                          & 2       & 17.64                                                    & 18.06                                                    & 53                                                                & 4                                                                 \\
Tox21\_AR  & 585                                                          & 2       & 17.99                                                    & 18.45                                                    & 53                                                                & 4                                                                 \\
Tox21\_ARE & 552                                                          & 2       & 17.01                                                    & 17.33                                                    & 53                                                                & 4                                                                 \\ \hline
\end{tabular}
\end{table*}

\subsection{Experimental Settings}
We evaluated the performance of EGIN and its two variants, EGIN-C and EGIN-E, on selected benchmark graph datasets. Our framework was implemented in Pytorch using Python version 3.9. EGIN, EGIN-C and EGIN-E were implemented according to the Equations \ref{EGIN}, \ref{cross}, and \ref{EGIN-E}, respectively. For each framework type, we implemented two versions: one retaining the $\epsilon$ parameter denoting as $-\epsilon$, and another without it in aggregation and update processes. We performed 10-fold cross-validation on graph classification tasks across each dataset to obtain the average and standard deviation of accuracies. The dimensions of hidden layers were chosen from \emph{\{32,64,128\}}. For the EGIN-E model, the dimensions of edge feature embeddings are selected from \emph{\{8,16,32\}}. Additionally, we conducted 10-fold cross validations on these datasets using other state-of-the-art models, including DGCNN \cite{zhang2018end}, GAT \cite{velivckovic2017graph}, GIN \cite{xu2018powerful} and EGCN \cite{gong2019exploiting}, as baselines for graph classification tasks.

\subsection{Experimental results}
Tables \ref{result1_table} and \ref{result2_table} show experimental results and comparisons with other models. 
\begin{table*}[htb]
\small
\centering
\caption{Experimental results 1}
\label{result1_table}
\begin{tabular}{lllllll}
\hline
                  & MUTAG           & AIDS            & COX2\_MD        & Cuneiform        & ER\_MD          & PTC\_FM          \\ \hline
DGCNN             & 84.2 $\pm$ 9.4 & 98.3 $\pm$ 0.9 & 55.1 $\pm$ 5.7 & 38.0 $\pm$ 7.9  & 73.4 $\pm$ 8.0  & 54.1 $\pm$ 5.7 \\
GAT               & 83.8 $\pm$ 11.1 & 94.6 $\pm$ 1.4 & 62.4 $\pm$ 10.4 &       76.9 $\pm$ 8.4   & 58.5 $\pm$ 6.9 & 59.0 $\pm$ 9.2     \\
GIN               & 87.4 $\pm$ 6.7  & 93.4 $\pm$ 1.8  & 52.8 $\pm$ 6.8 & 21.2 $\pm$ 3.9 & 59.7 $\pm$ 7.1 & 55.2 $\pm$ 10.2  \\
EGCN              & 87.9 $\pm$ 9.6 & 94.8 $\pm$ 1.2 & 63.7 $\pm$ 8.9  & 77.5 $\pm$ 10.1 & 78.6 $\pm$ 8.9 & 59.1 $\pm$ 7.0  \\ \hline
EGIN              & \textbf{92.5 $\pm$ 3.1}  & \textbf{99.7 $\pm$ 0.2}  & 68.1 $\pm$ 6.2 & 87.3 $\pm$ 3.4  & 75.9 $\pm$ 6.4 &  62.4 $\pm$ 9.2 \\
EGIN-$\epsilon$   & 90.2 $\pm$ 3.5  & 99.6 $\pm$ 0.2  & 65.1 $\pm$ 8.4 & 89.3 $\pm$ 4.7  & 77.5 $\pm$ 5.1  &  63.1 $\pm$ 7.6 \\
EGIN-C            & 92.8 $\pm$ 2.7  & 99.4 $\pm$ 0.3  & \textbf{74.2 $\pm$ 5.1} & 88.6 $\pm$ 5.5  & 78.0 $\pm$ 4.9 & 62.6 $\pm$ 8.2   \\
EGIN-C-$\epsilon$ & 91.1 $\pm$ 2.1  & 99.6 $\pm$ 0.4  & 73.4 $\pm$ 5.4  & 89.0 $\pm$ 4.7  & \textbf{80.7 $\pm$ 4.2} &  62.9 $\pm$ 7.5   \\
EGIN-E            & 89.0 $\pm$ 4.1 & 99.6 $\pm$ 0.3  & 67.1 $\pm$ 7.4 & 91.3 $\pm$ 4.8  & 76.4 $\pm$ 4.8 &  62.1 $\pm$ 5.6   \\
EGIN-E-$\epsilon$ & 88.8 $\pm$ 4.2 & 99.6 $\pm$ 0.3 & 64.1 $\pm$ 7.9  & \textbf{91.7 $\pm$ 3.8}  & 75.3 $\pm$ 5.0  & \textbf{63.4 $\pm$ 7.7}           \\ \hline
\end{tabular}
\end{table*}

\begin{table*}[htb]
\small
\caption{Experimental results 2}
\centering
\label{result2_table}
\begin{tabular}{lllllll}
\hline
                  & PTC\_FR          & PTC\_MM         & PTC\_MR         & Tox21\_AhR       & Tox21\_AR       & Tox21\_ARE      \\ \hline
DGCNN             & 58.5 $\pm$ 6.3  & 54.2 $\pm$ 6.9 & 54.7 $\pm$ 6.1 & 83.7 $\pm$ 3.5  & 96.5 $\pm$ 1.8 & 74.5 $\pm$ 4.1 \\
GAT               & 67.1 $\pm$ 9.1  & 65.0 $\pm$ 8.2 & 56.6 $\pm$ 8.3 & 87.4 $\pm$ 3.7  &  96.9 $\pm$ 1.4  & 82.3 $\pm$ 3.3\\
GIN               & 61.8 $\pm$ 10.4 & 60.4 $\pm$ 9.5 & 52.0 $\pm$ 9.9  & 82.8 $\pm$ 5.8 & 94.8 $\pm$ 2.3 & 81.7 $\pm$ 2.5 \\
EGCN              & 67.5 $\pm$ 9.1  & \textbf{66.0 $\pm$ 8.6} & 58.4 $\pm$ 5.8 & 86.4 $\pm$ 3.7 & 96.6 $\pm$ 0.7 & 82.5 $\pm$ 4.1 \\ \hline
EGIN              & 66.0 $\pm$ 4.2  & 62.3 $\pm$ 7.2 & 60.1 $\pm$ 6.3 & 88.1 $\pm$ 4.4 & 96.8 $\pm$ 1.3 &  83.2 $\pm$ 4.5 \\
EGIN-$\epsilon$   & 66.5 $\pm$ 6.4  & 63.0 $\pm$ 8.5 & 60.9 $\pm$ 7.7 & 88.6 $\pm$ 4.5   & 97.6 $\pm$ 1.5 & 83.9 $\pm$ 4.2 \\
EGIN-C            & 67.6 $\pm$ 4.8  &  63.7 $\pm$ 7.1  & 60.1 $\pm$ 5.5 & 88.3 $\pm$ 4.2  & \textbf{98.1 $\pm$ 1.2}  &  84.2 $\pm$ 4.9 \\
EGIN-C-$\epsilon$ & \textbf{68.9 $\pm$ 5.3}  & 65.6 $\pm$ 6.4 & \textbf{62.2 $\pm$ 4.2} & 87.4 $\pm$ 4.9  & 97.9 $\pm$ 1.4 & \textbf{84.6 $\pm$ 4.7}  \\
EGIN-E            & 66.2 $\pm$ 5.1  & 64.8 $\pm$ 8.4 & 61.0 $\pm$ 8.3 &      88.3 $\pm$ 5.8   &  96.6 $\pm$ 1.8 & 84.2 $\pm$ 4.9  \\
EGIN-E-$\epsilon$ & 66.5 $\pm$ 5.7  & 64.9 $\pm$ 8.1 & 59.4 $\pm$ 8.6 & \textbf{89.4 $\pm$ 4.5}   &  96.9 $\pm$ 1.7 & 84.4 $\pm$ 4.2   \\ \hline
\end{tabular}
\end{table*}
Among four baseline models, EGCN demonstrates the best performance across most datasets due to its incorporation of edge features. Experimental results reveal that EGIN and its variants outperform the baselines on all datasets except for \textbf{PTC\_MM}. Notably, there are substantial performance improvements on the \textbf{COX2\_MD} and \textbf{Cuneiform} datasets. However, the distinctions within EGIN and its variants are less clear. EGIN-C and EGIN-C-$\epsilon$ outperform other models on the \textbf{COX2\_MD} dataset, while EGIN-E and EGIN-E-$\epsilon$ exhibit superior classification performance on \textbf{Cuneiform} dataset. 
Intuitively, models retaining the $\epsilon$ parameter should exhibit better performance, as introducing the $\epsilon$ parameter may allow models to learn additional information. However, the experimental results indicate no clear conclusion on the accuracy performance of utilizing the $\epsilon$ parameter, as models with the $\epsilon$ parameter only outperform models without it on some datasets, rather than all of them. Furthermore, we discovered that introducing the $\epsilon$ parameter to EGIN-C and EGIN-E models deteriorates computational performances and significantly extends training time, raising concerns about the effectiveness of the $\epsilon$ parameter. 
Regardless of the $\epsilon$ parameter, we conclude that the increase in discriminative power from E-WL is validated by comparing performances of all modes. EGIN and its variants demonstrate better performance than other baseline models, including the original GIN model, across almost all datasets, highlighting the contributions of exploiting edge features.

\section{Conclusions}
In this paper, we proposed an Edged Weisfeiler-Lehman algorithm that extends the traditional Weisfeiler-Lehman algorithm by incorporating edge features in graphs. Additionally, we developed an EGIN model based on our proposed algorithm. To address a potential limitation related to the edge feature dimension in the original EGIN, we introduced two variants of EGIN. Experimental results demonstrate that our proposed models consistently outperform selected state-of-the-art baseline models, highlighting the effectiveness of leveraging edge features for graph classification tasks.

\bibliographystyle{unsrt}  
\bibliography{references}  

\appendix
\section{Exploration on a variant of E-WL}
In our proposed E-WL algorithm, we exclusively update representations of nodes at each iteration while keeping the edge representations unchanged. Intuitively, updating edge representations based on the representations of the two connecting nodes could further enhance the discriminative power of the E-WL algorithm. Building upon this idea, we developed a variant called Edged Weisfeiler-Lehman algorithm with Edge Aggregation (E-WL-EA), presented as Algorithm \ref{E-WL-EA algorithm}. The formulas for updating the representations of nodes and edges are shown in Equations \ref{E-WL-EA core1} and \ref{E-WL-EA core2}, where $d^{(l)}_{i,j}$ denotes the representation of edge $e_{i,j}$ at the $l^{th}$ iteration.
\begin{algorithm}[htb]	
	\caption{Edged Weisfeiler-Lehman algorithm with Edge Aggregation (E-WL-EA)}
	\label{E-WL-EA algorithm}
\textbf{Input: $\mathcal{G} = (\mathcal{V},\mathcal{E}, \mathcal{X}_{V}, \mathcal{X}_{E})$}\\
 $c^{(0)}_i = x^{n}_{i}$ for all $n_{i} \in \mathcal{V}$\\
 $d^{(0)}_{i,j} = x^{e}_{i,j}$ for all $e_{i,j} \in \mathcal{E}$\\
$\mathcal{X}^{(0)}_{V} = \mathcal{X}_{V}$\\
$\mathcal{X}^{(0)}_{E} = \mathcal{X}_{E}$\\
    \textbf{repeat}\\
        \Indp
        for each $n_{i} \in \mathcal{V}$\\
        \Indp
        $T_{i}^{(l)} = \{\{(c^{(l-1)}_{j},d^{(l-1)}_{i,j})| n_{j} \in \mathcal{N}_{(i)}\}\}, c^{(l-1)}_{j} \in \mathcal{X}^{(l-1)}_{V}, d^{(l-1)}_{i,j} \in \mathcal{X}^{(l-1)}_{E}$\\
        $c^{(l)}_i = HASH(c^{(l-1)}_i, T_{i}^{(l)})$\\
        \Indm
        for each $e_{i,j} \in \mathcal{E}$\\
        \Indp
        $d^{(l)}_{i,j} = HASH(d^{(l-1)}_{i,j},c^{(l)}_i, c^{(l)}_j)$\\
        \Indm
        $\mathcal{X}^{(l)}_{V} =  \{\{c^{(l)}_i|n_{i} \in \mathcal{V}\}\}$\\
        $\mathcal{X}^{(l)}_{E} =  \{\{d^{(l)}_{i,j}|e_{i,j} \in \mathcal{E}\}\}$\\
        \Indm
    \textbf{until} $c^{(l)}_i = c^{(l-1)}_i$ for all $n_{i} \in \mathcal{V}$\\
    \textbf{return} $\mathcal{X}^{(*)}_{V} = \mathcal{X}^{(l)}_{V}$
	
\end{algorithm}

\begin{equation}
\label{E-WL-EA core1}
    c^{(l)}_i = HASH(c^{(l-1)}_i, \{\{(c^{(l-1)}_{j},d^{(l-1)}_{i,j})| n_{j} \in \mathcal{N}_{(i)}\}\})
\end{equation}
\begin{equation}
\label{E-WL-EA core2}
    d^{(l)}_{i,j} = HASH(d^{(l-1)}_{i,j},c^{(l)}_i, c^{(l)}_j)
\end{equation}

Intuitively, the discriminative power of E-WL-EA might seem superior to that of the E-WL algorithm due to the inclusion of updated edge representations during the aggregation process. However, it can be rigorously proven that edge aggregation is unnecessary and does not enhance discriminative power. As a result, the discriminative power of E-WL-EA is equivalent to that of E-WL, indicating that E-WL is an efficient algorithm.
\begin{theorem}
\label{t2}
    Discriminative power of E-WL-EA remains the same as that of E-WL.
\end{theorem}

To prove the theorem \ref{t2}, we first introduce a Lemma \ref{lemma_t2} here. Intuitively, we can always map the nodes to identical representations if two algorithms possess the same discriminative power at current iterations on a graph. 
\begin{lemma}
\label{lemma_t2}
Suppose for any pair of edges $e_{i,j}$ and $e_{v,u}$ in a graph $\mathcal{G} =(\mathcal{V},\mathcal{E}, \mathcal{X}_{V},$ $\mathcal{X}_{E})$, $d^{(l-1)}_{i,j} = d^{(l-1)}_{v,u}$ if and only if $x^{e}_{i,j} = x^{e}_{v,u}$, there always exist two hash functions $c^{(l)}_i = HASH1(c^{(l-1)}_i, \{\{(c^{(l-1)}_j,d^{(l-1)}_{i,j})_{tup}| n_{j} \in \mathcal{N}_{(i)}\}\})$ and $c^{(l)}_v = HASH2(c^{(l-1)}_v,\{\{(c^{(l-1)}_u,d^{(l-1)}_{v,u})_{tup}| n_{u} \in \mathcal{N}_{(v)}\}\})$, such that $c^{(l)}_i = c^{(l)}_v$.
\end{lemma}

\begin{proof}
This theorem is proved by contradiction. Given two graphs $\mathcal{G}^{1} = (\mathcal{V}^{1},\mathcal{E}^{1}, \mathcal{X}^{1}_{V},\mathcal{X}^{1}_{E})$ and $\mathcal{G}^{2} = (\mathcal{V}^{2},\mathcal{E}^{2}, \mathcal{X}^{2}_{V},\mathcal{X}^{2}_{E})$, we assume E-WL-EA is more powerful than E-WL. This assumption implies that there must exist a case where following two conditions are satisfied.
\begin{equation}
\begin{split} 
    (c^{(l-1)}_i, \{\{(c^{(l-1)}_j,d^{(l-1)}_{i,j})| n_{j} \in \mathcal{N}_{(i)}, e_{i,j} \in \mathcal{E}^{1}\}\}) \\ \neq  (c^{(l-1)}_v, \{\{(c^{(l-1)}_u,d^{(l-1)}_{v,u})| n_{u} \in \mathcal{N}_{(v)},  e_{v,u} \in \mathcal{E}^{2}\}\})
\end{split}
\end{equation}
\begin{equation}
\begin{split} 
    (c^{(l-1)}_i, \{\{(c^{(l-1)}_j,x^{e}_{i,j})| n_{j} \in \mathcal{N}_{(i)}, e_{i,j} \in \mathcal{E}^{1}, x^{e}_{i,j} \in \mathcal{X}_{E}^{1}\}\}) \\ =   (c^{(l-1)}_v, \{\{(c^{(l-1)}_u,x^{e}_{v,u})| n_{u} \in \mathcal{N}_{(v)}, e_{v,u} \in \mathcal{E}^{2}, x^{e}_{v,u} \in \mathcal{X}_{E}^{2}\}\})
\end{split}
\end{equation}
For simplicity, we assume there is only one Node-Edge tuple in each multiset. By simplifying these two conditions, we derive four simpler conditions.
\begin{equation}
\label{thm2condition}
 \begin{cases}
    d^{(l-1)}_{i,j} \neq d^{(l-1)}_{v,u}\\
    x^{e}_{i,j} = x^{e}_{v,u}\\
    c^{(l-1)}_i = c^{(l-1)}_v\\
    c^{(l-1)}_j = c^{(l-1)}_u
    \end{cases}
\end{equation}
Here we aim to prove that all four conditions can not be satisfied simultaneously. Assume $d^{(l-1)}_{i,j} \neq d^{(l-1)}_{v,u}$ is true, and the update equations of $d^{(l)}_{i,j}$ and $d^{(l)}_{v,u}$ in E-WL-EA algorithm are given by: 
\begin{equation}
\label{edge_update}
    d^{(l-1)}_{i,j} = HASH(d^{(l-2)}_{i,j},c^{(l-1)}_i,c^{(l-1)}_j)
\end{equation}
\begin{equation}
    d^{(l-1)}_{v,u} = HASH(d^{(l-2)}_{v,u},c^{(l-1)}_v,c^{(l-1)}_u)
\end{equation}
Therefore, at least one of these three conditions should be true in this case:
\begin{equation}
 \begin{cases}
    d^{(l-2)}_{i,j} \neq d^{(l-2)}_{v,u}\\
    c^{(l-1)}_i \neq c^{(l-1)}_v\\
    c^{(l-1)}_j \neq c^{(l-1)}_u
    \end{cases}
\end{equation}
The later two conditions $c^{(l-1)}_i \neq c^{(l-1)}_v$ and $c^{(l-1)}_j \neq c^{(l-1)}_u$, conflict with formula \ref{thm2condition} and thus cannot be true. Next, we need to prove that $d^{(l-2)}_{i,j} \neq d^{(l-2)}_{v,u}$ cannot be satisfied in this case. Since the update process for node representations is injective, if $c^{(l-1)}_i = c^{(l-1)}_v$ and $c^{(l-1)}_j = c^{(l-1)}_u$, then $c^{(k)}_i = c^{(k)}_v$ and $c^{(k)}_j = c^{(k)}_u$ should hold for all $ k = 0, 1, 2, ..., l-2$. Knowing that $x^{e}_{i,j} = x^{e}_{v,u}$, and according to Lemma \ref{lemma_t2}, we can always construct injective functions to ensure that $d^{(k)}_{i,j} = d^{(k)}_{v,u}$ for all $ k = 1, 2, ..., l-2$ as well, which contradicts condition $d^{(l-2)}_{i,j} \neq d^{(l-2)}_{v,u}$. Therefore, the four conditions in Formula \ref{thm2condition} cannot all be satisfied simultaneously, making the assumption impossible. Consequently, the discriminative power of E-WL-EA is not superior to that of E-WL.
\end{proof}

\end{document}